\documentclass{jmlr} 
\usepackage{graphicx} 

\usepackage{authblk}

\usepackage{amsfonts}       
\usepackage{amsmath}
\usepackage[dvipsnames]{xcolor} 
\usepackage{thm-restate}
\usepackage{hyperref}       
\usepackage{venndiagram}
\usepackage{caption}
\usepackage{subcaption}
\usepackage{wrapfig}
\usepackage{tabulary}
\usepackage{booktabs}
\usepackage{algorithm}
\usepackage{algpseudocode}
\usepackage{times}
\usepackage{tikz}
\usetikzlibrary{decorations.pathreplacing,angles,quotes}
\usetikzlibrary{patterns}

\definecolor{darkcyan}{rgb}{0.0, 0.55, 0.55}

\newtheorem{claim}[theorem]{Claim}

\newcommand{\SP}{\mathit{SP}}

\renewcommand{\setminus}{\backslash}

\title{Distance-based Learning 
of Hypertrees}
\author[1]{Shaun Fallat}
\author[2]{Kamyar Khodamoradi}
\author[3]{David Kirkpatrick}
\author[2]{Valerii Maliuk}
\author[4]{S.\ Ahmad Mojallal}
\author[2,5]{Sandra Zilles}
\affil[1]{Department of Mathematics and Statistics, University of Regina}
\affil[2]{Department of Computer Science, University of Regina}
\affil[3]{Department of Computer Science, University of British Columbia}
\affil[4]{Department of Computer Science, Simon Fraser University}
\affil[5]{Alberta Machine Intelligence Institute (Amii)}
\date{}

\begin{document}

\maketitle

\begin{abstract}%
  We study the problem of learning hypergraphs with shortest-path queries ($\SP$-queries), and present the first provably optimal online algorithm for a broad and natural class of hypertrees which we call \emph{orderly hypertrees}. Our online algorithm can be transformed into a provably optimal offline algorithm.
  Orderly hypertrees can be positioned within the Fagin hierarchy of acyclic hypergraph (well-studied in database theory), and strictly encompass the broadest class in this hierarchy that is learnable with subquadratic $\SP$-query complexity. \\
  Recognizing that in some contexts, such as evolutionary tree reconstruction, distance measurements can degrade with increased distance, we also consider a learning model that uses bounded distance queries. In this model, we demonstrate asymptotically tight complexity bounds for learning general hypertrees.
  
\end{abstract}


\section{Introduction}


Due to their various applications, hypergraphs -- sometimes referred to as set systems or range spaces -- have been the focus of numerous learning-theoretic studies in the past two decades. Typical learning settings are rooted in similar studies for conventional graphs (hypergraphs with edges of size 2), and often involve models in which a learner interacts with an information source (oracle) by asking queries about the unknown target hypergraph.
The settings differ in (i) the type of query the learner can ask, (ii) whether the learner is deterministic or randomized, and (iii) whether it learns adaptively 
or non-adaptively \citep{AbrahamsenBRS16,Hein,Janardhanan17,9,10,Reyzin}. 

One well-developed line of research has focused on the reconstruction of evolutionary (phylogenetic) trees (rooted edge-weighted trees whose leaves are labeled by species). Reconstruction in this case involves determining the branching structure of the tree based on absolute or relative distances between species. 
\citet{Hein} studied evolutionary tree reconstruction using 
shortest-path queries ($\SP$-queries), which allow the reconstruction algorithm (learner) to specify two leaves whose distance in the target tree is revealed by an oracle. 
Hein showed that evolutionary trees can be 
learned in this way using $O(n\log n)$ $\SP$-queries,
where $n$ denotes the number of leaves in the target tree.
\cite{Brodal2001} studied evolutionary tree reconstruction using 
weaker relative-distance queries (introduced by \cite{Kannan1996}), and showed that 
$O(n\Delta\log_\Delta n)$ such queries suffice to reconstruct a target tree, where $\Delta$ denotes the maximum degree of vertices in the target tree.
\cite{10} showed that this bound is tight, up to a constant multiplicative factor, by demonstrating an $\Omega(n\Delta\log_\Delta n)$ lower bound on the number of $\SP$-queries required in the worst case. 
They also introduced and analysed variants of evolutionary tree reconstruction using both bounded and approximate distance queries.

Recently, \citet{BastideG25} presented a simple
algorithm, based on structured insertion of vertices, that reconstructs arbitrary unweighted trees.  
Specifically, they 
showed that $O(n\Delta\log_\Delta n)$ 
(unconstrained) $\SP$-queries 
suffice to 
reconstruct 
any tree on $n$ vertices, where $\Delta$, the degree of the target tree,
is not known to the 
reconstruction algorithm in advance. 
In addition, they provide a matching lower bound, even for the expected number of $\SP$-queries needed by any randomized reconstruction algorithm.

The main question pursued in our paper is whether interesting classes of \emph{hyper}trees---hyper\-graphs that have many of the properties of trees---can be learned with $o(n^2)$ $\SP$-queries (in the worst case over all hypertrees in the class). This question is not only natural to ask but also of relevance to studies in database theory, computational biology, and other application domains in which hypergraphs (and specifically hypertrees) are used for modeling relationships between entities. To the best of our knowledge, the only existing work on learning hypertrees with $\SP$-queries is a recent paper by \citet{FMMZ}; they provide two algorithms, each of which yields subquadratic query complexity only for a very restricted subclass of hypertrees. While Fallat et al.\ proved that learning hyperstars (hypertrees of diameter two) already requires $\Omega(n^2)$ $\SP$-queries in the worst case, their algorithms suggest that there is hope for efficient learning of a large subclass of hypertrees of diameter at least three.

Hypertrees can be defined in several different ways, which draw parallels with characterizations of trees. The equivalence of these definitions is not always immediate~\citep{Brandstadt99}. 
In database theory,  an entire hierarchy of notions of hypergraph acyclicity has been proposed \citep{Fagin1983Degrees}, dubbed $\alpha$-, $\beta$-, and $\gamma$-acyclicity, as well as Berge-acyclicity (in increasing order of restriction). We demonstrate that the class of $\gamma$-acyclic hypertrees, even when restricted to those of diameter at least three, has an $\SP$-query complexity of $\Omega(n^2)$. By contrast, we define a new class of hypertrees, called orderly hypertrees, and show that this class (i)~fits into Fagin's hierarchy between $\gamma$-acyclic and Berge-acyclic hypergraphs, and (ii)~allows for efficient $\SP$-query learning, when restricting to hypergraphs of diameter at least three.

We present an algorithm showing that $O(n\Delta\log_\Delta m)$ $\SP$-queries
suffice to learn the class of orderly hypertrees 
of diameter at least three, 
where $\Delta$ is the maximum number of edges any single edge in the target hypertree can intersect, and $m$ is the number of edges in the target hypertree (both unknown to the learner). 
The algorithm follows the incremental approach of Hein's algorithm for learning phylogenetic trees, with several essential differences stemming from the facts that (i) the learning target is a hypertree, whose edges can connect more than two vertices, and (ii) $\SP$-queries can involve non-leaf vertices. 

The efficiency of our algorithm takes advantage of a tree separator argument, similar in style to arguments previously used in the context of learning evolutionary and conventional trees  \citep{Hein,Brodal2001,BastideG25}, but requiring novel insights in order to handle hypertrees.

Our algorithm exploits the fact that every hypergraph has a unique representation as a labelled bipartite graph, which we call its \emph{skeleton graph}. Notably, orderly hypertrees are exactly those hypergraphs whose skeletons are trees. 
The algorithm learns the skeleton tree of the target, and thus the target hypertree itself.

Our algorithm is most naturally formulated as an \emph{online}\/ algorithm that inserts vertices one by one into an initially empty hypergraph, incrementally constructing the unique sub-hypergraph consistent with the queries posed so far. This makes it amenable to applications in which entities (vertices) are revealed in a stream, and partial reconstruction might suffice. 

An argument inspired by the lower bound arguments used by \cite{10} and \citet{BastideG25} for evolutionary and conventional trees allows us to show that an off-line variant of our algorithm is, up to a constant factor, optimal in the worst case, even compared with algorithms that know $m$ and $\Delta$ in advance. 
As such, our algorithm provides a substantial improvement over the two methods presented by \citet{FMMZ}. 
It follows, as well, that the online formulation of our algorithm is also optimal in terms of worst-case query complexity among all possible online algorithms for the same task.

In the context of evolutionary tree reconstruction, \cite{10} observe that measures of the distance between species degrade as the distance increases. 
This motivates the consideration of bounded distance queries, that provide accurate distance information only when the distance between specified species is less than some fixed threshold. 
We consider a similar query restriction in the context of hypertree learning, and describe query algorithms and lower bounds that together provide an asymptotically tight bound on bounded query complexity, for all fixed distance thresholds.

\section{Preliminaries}

A hypergraph $H=(V,E)$  consists of a finite vertex set $V$ of size $n$ and an edge set $E$ of size $m$; elements of $E$ are subsets of $V$ of cardinality $\ge 2$. 
  The \textit{line graph} of hypergraph $H$, denoted $L(H)$, is the graph with vertex set $E$ and edge set 
 $\{ (e, e') \;|\; e, e' \in E  \; \& \; e\cap e'\neq \emptyset \}$.

 The \emph{edge degree}\/ of an edge $e$ in $H$ is the number of edges $e'\ne e$ in $H$ for which $e\cap e'\ne\emptyset$.
 A \emph{private vertex}\/ of an edge $e$ in  $H$ is a vertex in $e$ that does not belong to any edge $e'\ne e$ in $H$. We denote by $P_H(e)$ the set of all \emph{private vertices} of edge $e$. 

A \textit{path} $P$ in $H$ from $v_1$ to $v_{t}$ is an alternating sequence $v_1e_1\cdots  e_{t-1}v_{t}$ in which  $v_1$,\dots, $v_{t}\in V$ are distinct vertices, $e_1,  \ldots,  e_{t-1}\in E$ are distinct edges, and for $1 \le i < t$, $v_{i}, v_{i+1} \in e_i$. The \emph{length} of $P$ is $t-1$. 
Further, the \textit{distance} $d_H(v, w)$ 
between two vertices $v$ and $w$ of $H$ is the minimum length of a path from $v$ to $w$, which is $\infty$ if no such path exists. 
Hypergraph $H$ is {\em connected} if, for any two vertices $u$, $v$ in $H$, there is a path from $u$ to $v$. Given $H$, the eccentricity of a vertex $v\in V$ is given by $\max_{w\in V} d_H(v,w)$. 

 If $U \subseteq V$, we denote by $H[U]$ the 
sub-hypergraph of \( H \) induced on $U$:
$H[U] = \{ e \cap U \mid e \in E \} \setminus \{\emptyset\}$.

 \begin{definition}\label{def:hypertree}
A hypergraph $H$ is called a \emph{hypertree} if there exists a tree $T$ such that every hyperedge of $H$ is the set of vertices of a connected subtree of $T$.
\end{definition}

 \begin{remark}\label{rem:hypertree}
 Note that, unlike trees, hypertrees can contain cycles (paths with identical endpoints). 
 However, cycles in hypertrees are very constrained, as captured in the following equivalent characterization
 \citep{Brandstadt99}:
   A connected hypergraph $H$ is a \textit{hypertree} if its line graph $L(H)$ is chordal\footnote{A graph $G$ is \textit{chordal} if every induced cycle in $G$ has exactly three vertices.} 
   and its edge set $E$ satisfies the \emph{Helly property}: for every subset $S \subseteq E$, if every two edges in $S$ have a nonempty intersection, then $S$ has a nonempty intersection.
\end{remark}

The focus of this paper is on learning classes of hypergraphs with distance queries. 
In this context, a learner for a class $\mathcal{H}$ of 
hypergraphs is an algorithm that works iteratively in rounds.  In each round, it asks a query about an unknown target hypergraph $H=(V,E)\in\mathcal{H}$ and receives the correct answer from an oracle.  
The learning process stops once the target $H$ is the only  hypergraph in $\mathcal{H}$ that is consistent with the answers to all 
queries. 
The learner is said to \emph{learn}\/ the class $\mathcal{H}$ if it successfully identifies every target $H\in\mathcal{H}$ in this fashion, see, e.g.,~\citep{Beerliova}.

Unless stated otherwise, we assume that the learner knows the vertex set $V$, but has to identify the edge set $E$ (as a set of subsets of $V$). It does not know any parameters of the target $H$ (such as the number $m$ of edges, the diameter $d$, etc.) unless pre-specified by $\mathcal{H}$. The efficiency of the learner is assessed in terms of the number of rounds of the learning process (i.e., the number of queries asked) in the worst case considered over all possible target hypergraphs in $\mathcal{H}$.
Learning in our setup is \emph{adaptive}\/ in that each query may depend on the queries and responses processed in previous rounds.
Our main result is an efficient algorithm for learning a natural class of hypertrees (to be defined below) from \emph{shortest path queries}\/ ($\SP$-queries, for short). An $\SP$-query consists of a pair $(v,w)\in V^2$ of vertices and is answered with $d_H(v,w)$, where $H$ is the target hypergraph.

\begin{definition}
    A class $\mathcal{H}$ is \emph{$\SP$-learnable}, if there exists a learner that learns $\mathcal{H}$ using $\SP$-queries. $\mathcal{H}$ is \emph{hard to learn with $\SP$-queries}, if every learner that learns $\mathcal{H}$ with $\SP$-queries uses $\Omega(n^2)$ queries in the worst case over all targets $H\in\mathcal{H}$, where $n$ is the number of vertices in $H$. 
\end{definition}

\begin{remark}
    It is straightforward to see that, in learning a class of hypergraphs with multiple connected components using $\SP$-queries, an algorithm must first identify the associated partition of the vertex set, and then discover the structure of the individual components.
    To keep separate these sources of query complexity,
    we restrict our attention to learning families of connected hypergraphs.
\end{remark}

\begin{remark} Note that $\mathcal{H}$ is $\SP$-learnable if and only if $\mathcal{H}$ contains no two distinct members $H,H'$ such that $d_H=d_{H'}$. It follows immediately from this that if $\mathcal{H}$ is $\SP$-learnable then $O(n^2)$ $\SP$-queries suffice in the worst case.\\
Figure~\ref{fig:gammahard}(left) displays two hypertrees that induce the same distance function; any class containing them both is hence not $\SP$-learnable. $\SP$-learnability is also impossible when $\mathcal{H}$ has two members $H,H'$, where $H$ has nested edges (i.e., at least two edges, one of which is strictly contained in the other) and $H'$ results from $H$ by removing at least one edge that is contained in another edge. We therefore assume throughout this paper that a hypergraph does not have nested edges.
\end{remark}

\begin{remark}\label{rem:SPhard}
 \cite{FMMZ} showed that the class of hyperstars (hypertrees of diameter two) is hard to learn with $\SP$-queries. 
 However, 
 even $\SP$-learnable classes of hypertrees of diameter at least three can be hard to learn with $\SP$-queries. 
 As we argue in the proof of the following lemma, an adversary can force a learner to make quadratically many $\SP$-queries in order to learn all hypertrees in 
 any class $\mathcal{H}$ containing all hypertrees isomorphic to the hypertree in Figure~\ref{fig:gammahard}(right).  
\end{remark}

\begin{lemma}\label{lem:nonorderly}
    Any class $\mathcal{H}$ containing all hypertrees isomorphic to the structure $H_n$ in Figure~\ref{fig:gammahard}(right) is hard to learn with $\SP$-queries.
\end{lemma}

\begin{proof} Note that 
$|e_2\cap e_j|=1$ for $j\in \{1, 3, 4\}$, and the intersection of $e_3$ and $e_4$ consists of $n-5$ vertices, one of which belongs to $e_2$.


The vertices in $(e_3 \cup e_4)\setminus e_2$ cannot be distinguished from one another by their distance from any of the other vertices. Furthermore, the private vertices in $e_3$ and $e_4$ cannot be distinguished from the other vertices in $(e_3 \cup e_4)\setminus e_2$ by their distance from  the other vertices in $(e_3 \cup e_4)\setminus e_2$.  It follows that to identify the private vertices in $e_3$ and $e_4$, the unique pair of vertices in $(e_3 \cup e_4)\setminus e_2$ at distance two from one another, an algorithm must make an $\SP$-query of this pair.
Since an adversary can force this pair to be the last pair of vertices from $(e_3 \cup e_4)\setminus e_2$ to be queried, it follows that at least $(n-4)^2$ $\SP$-queries are required, in the worst case, to learn hypertrees in $\mathcal{H}$.
\end{proof}

\begin{figure}[t]
  \centering
  \includegraphics[width=0.4\textwidth, scale=0.2]{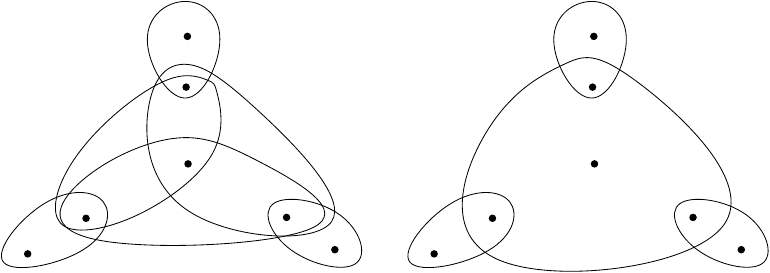}
  \hspace*{2cm}\includegraphics[width=0.3\textwidth, scale=0.02]{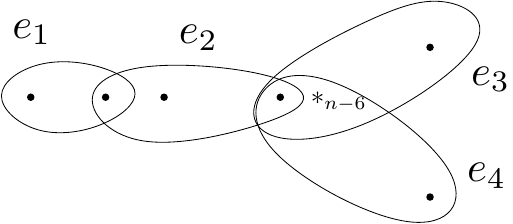}
 \caption{(Left) Two hypertrees with identical distances between corresponding vertices; vertices displayed in the same position in the left and right trees are assumed to be identical. (Right) Any class $\mathcal{H}$ containing all hypertrees isomorphic to this hypertree, with diameter 3, is hard to learn with $\SP$-queries. (Here $*_{n-6}$ denotes a cluster of $n-6$ vertices).}
  \label{fig:gammahard}
\end{figure}

\section{Orderly Hypergraphs}

The leftmost and rightmost hypertrees in Figure~\ref{fig:gammahard} have in common that three of their edges have a non-empty intersection while two of the three have a strictly larger intersection. 
It turns out that this can make $\SP$-learnability difficult. We hence propose the study of what we call \emph{orderly}\/ hypertrees.\footnote{Orderly hypertrees were previously also considered by \cite{FMMZ}, without explicitly naming them.}

\begin{definition}
    A hypergraph $H=(V, E)$ is intersection-orderly, or \emph{orderly}\/ for short, if, for any two distinct edges $e_1,e_2\in E$ with $S:=e_1\cap e_2\ne\emptyset$, and any edge $e\in E$, we have either $S\subset e$ or $S\cap e=\emptyset$.
\end{definition}

Our focus in this paper is on orderly hypertrees.
We will show that orderly hypertrees of diameter at least three can be learned efficiently using $\SP$-queries. This result is interesting from a learning-theoretic perspective, since both (i) the class of orderly hyperforests (collections of orderly hypertrees) and (ii) the class of orderly hypertrees of diameter two are hard\footnote{For (i), note that distinguishing between (a) a set of isolated vertices and (b) a set of isolated vertices plus an edge of size two requires $\Omega(n^2)$ $\SP$-queries. }
to learn with $\SP$-queries.

Equally important though, this result is significant from an application point of view since, unlike general hypertrees, the notion of orderly hypertree fits nicely into the Fagin hierarchy (cf. \citep{Fagin1983Degrees}) of acyclic hypergraphs, which is a familiar and well-studied in relational database theory. 
According to Fagin, a hypergraph $H=(V, E)$ is
\begin{itemize}
    \item \emph{acyclic} if there is no set $\{u_1, u_2, \ldots u_t \} \subseteq V$ such that
    $H[\{u_1, u_2, \ldots u_t \}]$ is an ordinary graph cycle: 
    $\{ \{u_1, u_2 \}, \{u_2, u_3 \}, \ldots, \{u_{t-1}, u_t \}, \{u_t, u_1 \} \}$.
   \item \emph{$\alpha$-acyclic} if and only 
   \( H \) is both acyclic and conformal
   \footnote{A hypergraph $H$ is \emph{conformal}\/ if for each set $U$ of vertices, if each pair of vertices in $U$ is contained in some edge of $H$, then $U$ itself is contained in some edge of $H$.}.
   
    \item \emph{$\beta$-acyclic} if and only if all subsets of $E$ are $\alpha$-acyclic.
    \item \emph{$\gamma$-acyclic} if and only if \( H \) is beta-acyclic and does not contain pairwise distinct vertices $x,y,z$ such that 
    $ \{\{x, y\}, \{y, z\}, \{x, y, z\}\} \subseteq H[\{x, y, z\}]$.
    \item \emph{Berge-acyclic} if the associated (bipartite) incidence graph $G = \left\{ \{x, e\} \;\middle|\; x \in e,\, e \in H \right\}$ is acyclic.
    Equivalently, \( H \) is $\beta$-acyclic and does not contain pairwise distinct vertices $x,y,z$ such that 
     $\{\{x, y\}, \{x, y, z\}\} \subseteq H[\{x, y, z\}]$.  
\end{itemize}

It turns out that the class of orderly hypertrees lies between that of $\gamma$-acyclic and that of Berge-acyclic hypergraphs: 

\begin{claim}
     A hypergraph $H=(V, E)$ is an orderly hypertree if and only if $H$ is $\beta$-acyclic and does not contain pairwise distinct vertices $x,y,z$ such that 
    $\{\{x, y\}, \{y\}, \{x, y, z\}\} \subseteq H[\{x, y, z\}]$.  
\end{claim}

\begin{proof}
``$\Rightarrow$'' Let $H=(V,E)$ be an orderly hypertree. 
Suppose, towards a contradiction, that $H$ is not $\beta$-acyclic. 
Thus, there must exist a subset $E' \subseteq E$ such that the hypergraph $H' = (V, E'))$ is not $\alpha$-acyclic. 
If $H'$ is not conformal then either $H'$ is not orderly or $E'$ does not satisfy the Helly property, contradicting our assumption that $H$ is an orderly hypertree. 
Hence, $H'$ must not be 
acyclic; i.e. there exist vertices $v_0, v_1, \ldots, v_{k-1}$ in $V$ and edges $e_0, e_1, \ldots, e_{k-1}$ in $E'$ such that, for all $i = 0, \ldots k-1$, $\{ v_i, v_{i+1} \} \in e_i$, $v_i \in e_{i-1} \cap e_i$, and $v_i \notin e_j$, for $j \notin \{ i-1, i \}$, (with indices reduced mod $k$).
Since $H'$ is a hypertree, $L(H')$ must be chordal, and so there exists $e_i \notin \{e_0, e_1 \}$ such that $e_i$ intersects both $e_0$ and $e_1$. 
Now since $H$ is orderly and $v_1 \in e_0 \cup e_1$ but $v_1 \notin e_i$, it follows that $e_0 \cap e_1 \cap e_i = \emptyset$ contradicting the assumption that $E'$ satisfies the Helly property.

Next, suppose, towards another contradiction, that 
$H$ contains three distinct vertices $x,y,z$ such that $\{\{x, y\}, \{y\},$ $\{x, y, z\}\}\subseteq H[\{x, y, z\}]$. If so, then $H$ would have pairwise distinct edges $e_1,e_2,e_3$ with $e_1\cap\{x,y,z\}=\{x,y\}$, $e_2\cap\{x,y,z\}=\{y\}$, and $e_3\cap\{x,y,z\}=\{x,y,z\}$. Let $S=e_1\cap e_3$. Since $\{x,y\}\subseteq S$, we have $S\ne \emptyset$. Now $e_2\cap S\ne\emptyset$, but $e_2$ does not contain $S$, contradicting the premise that $H$ is orderly. 

``$\Leftarrow$'' Suppose that $H=(V, E)$ is a $\beta$-acyclic hypergraph that does not contain three distinct vertices $x,y,z$ with $\{\{x, y\}, \{y\}, \{x, y, z\}\} \subseteq H[\{x, y, z\}]$. it suffices to show that $L(H)$ is chordal, $E$ satisfies the Helly property,
and $H$ is orderly.

To show that $L(H)$ is chordal, suppose by way of contradiction that $L(H)$ has a chordless cycle of length at least four. The set of vertices of such cycle are a set of edges in $H$ that form a sub-hypergraph of $H$ that is not $\alpha$-acyclic. Thus $H$ is not $\beta$-acyclic---a contradiction. Therefore, $L(H)$ is chordal.

To verify that $E$ 
is orderly, 
let $e_1$ and $e_2$ be two edges in $H$ with non-empty intersection.
Suppose, towards a contradiction, that there is an edge $e\in E\setminus \{e_1,e_2\}$ such that $e$ contains some element $y\in e_1\cap e_2$, but $e\not\supseteq e_1\cap e_2$. Thus, let $x\in e_1\cap e_2\setminus e$. Since we limit ourselves to hypergraphs without nested edges, there are vertices $z,z'$ such that $z\in e_1\setminus e_2$ and $z'\in e_2\setminus e_1$. By the premise, $\{\{x, y\}, \{y\}, \{x, y, z\}\} \not\subseteq H[\{x, y, z\}]$ and $\{\{x, y\}, \{y\}, \{x, y, z'\}\} \not\subseteq H[\{x, y, z'\}]$. Therefore, $e$ contains both $z$ and $z'$. Thus 
$\{e_1,e_2,e\}[\{x, z, z'\}]=\{\{x,  z\},\{x, z'\},\{z,z'\}\}$.
In particular, $\{e_1,e_2,e\}$ is not acyclic, and thus $H$ is not $\beta$-acyclic---a contradiction. 
Hence, $H$ is orderly.

Suppose, towards another contradiction, that $E$ does not satisfy the
Helly property. Let $S:=\{e_1,\ldots,e_k\}$ is any minimal subset of $E$ contradicting the Helly property, i.e., any two edges in $S$ intersect, $e_1\cap\cdots\cap e_{k-1}\ne\emptyset$, but $e_1\cap\cdots\cap e_k=\emptyset$. Since $H$ is orderly, $e_1\cap\cdots\cap e_{k-1}=e_i\cap e_j$ for any two distinct $i,i\in\{1,\ldots,k-1\}$. 
Let $x \in e_1 \cap e_2$, $y \in e_1 \cap e_k$, and $z \in e_2 \cap e_k$. 
Then $H'[ \{x, y, z \} ] = \{ \{x, y \}, \{ y, z \}, \{x, z \} \}$, where $H'= (V, \{ e_1, e_2, e_k \})$. 
In particular, $H'$ is not acyclic, and thus $H$ is not $\beta$-acyclic---a contradiction.
Hence, $E$ satisfies the Helly property.

\end{proof}

Note that the hypertree shown in Figure~\ref{fig:gammahard} (right) is $\gamma$-acyclic. Thus, by Remark~\ref{rem:SPhard}, the class of $\gamma$-acyclic hypertrees is hard to learn with $\SP$-queries. Orderly hypertrees (of diameter at least three) therefore strictly contain the structurally richest class in the acyclicity hierarchy that can be learned with $o(n^2)$ $\SP$-queries in the worst case, as we will demonstrate below.

\section{Skeletons of Hypergraphs}


Orderliness of a hypergraph $H=(V,E)$ implies that $V$ can be partitioned into disjoint subsets each of which is either the set of private vertices of an edge or the intersection of two edges.

\begin{proposition}\label{prop:orderlyPartition}
    For an orderly hypergraph $H = (V,E)$ the 
sets in $\bigcap^2_E = \{ e \cap e' \;|\; e, e' \in E, \; e\ne e', \;  e\cap e'\neq \emptyset \}$ together with the sets $P_H(e)$, for $e \in E$, form a partition of $V$. Any two vertices $v$ and $v'$ in the same part are \emph{equivalent} in the sense that if $v e_1 ...e_{t-1} x$ is a path in $H$ then $v' e_1 ...e_{t-1} x$ is a path in $H$ (so the distances from $v$ and $v'$ to any third vertex $x$ are identical). 
\end{proposition}

\begin{proof}
Clearly, $\bigcap^2_E\cup (\bigcup_{e\in E}P_H(e))=V$. Moreover, given $e\in E$, no vertex $v\in P_H(e)$ is contained in $P_H(e')$ for any $e'\in E\setminus\{e\}$ or in  $e' \cap e'' $ for any $e',e'' \in E$, $e'\ne e''$. Now suppose a vertex $v\in V$ belongs to two distinct sets $e_1 \cap e_1'\in \bigcap^2_E$, $e_2 \cap e'_2\in\bigcap^2_E$, where $e_i,e'_i\in E$. Then the intersection $e_1\cap e'_1\cap e_2\cap e'_2$ contains $v$, but $e_1\cap e'_1\ne e_2\cap e'_2$, which immediately violates the definition of orderliness. 
\end{proof}   

The disjoint partition property facilitates a helpful representation of an orderly hypergraph as a (conventional) bipartite graph.

\begin{definition}\label{def:skeleton}
Let $H(V, E)$ be any hypergraph. The \emph{skeleton graph} of $H$, denoted $S(H)$, is defined as the bipartite graph with the following properties.
\begin{itemize}
    \item Nodes in one part are black and in the other part are colored (either blue or red);
    \item Black nodes correspond to the elements of $E$, the edges of $H$. Blue nodes correspond to the non-empty sets of private vertices $P_H(e)$, $e \in E$. Red nodes correspond to the elements of $\bigcap^2_E$, the non-empty intersections of distinct edges in $E$; and
    \item Edges of $S(H)$ join 
    (a) blue nodes to their corresponding (black) edge, and 
    (b) red nodes to all (black) edges whose common intersection is the set associated with that red node.
\end{itemize}
\end{definition}

 Hypergraphs can be uniquely reconstructed from their skeleton graphs and vice versa, so that the problem of learning one is equivalent to the problem of learning the other.

\begin{lemma}
Let $H\!=\!(V, E)$, $H'\!=\!(V, E')$ be orderly hypergraphs. If $S(H) = S(H')$ then $H = H'$. 
\end{lemma}

\begin{proof} The black nodes of $S(H)$ correspond to the edges of $H$. The vertices forming the edge corresponding to a given black node are exactly the vertices associated with the colored nodes adjacent to that black node.
\end{proof}

Note that the color of nodes is implicit in the skeleton structure: blue nodes have degree one, and red nodes have degree at least two. 
For a given orderly hypergraph $H$, denote by $[v]$ the colored node of $S(H)$ that contains vertex $v$.

If $H$ is a hypergraph, a path $P$ joining two nodes in its skeleton graph $S(H)$ has length $\lambda(P)$ given by the number of black nodes in $P$. Accordingly, if $H$ is an orderly hypergraph and $v_1e_1v_2\cdots  e_{t-1}v_{t}$ is a path from vertex $v_1$ to vertex $v_t$ in $H$, then $P=[v_1]e_1[v_2] \cdots  e_{t-1}[v_{t}]$ and
$P'=[v_1]e_1[v_2] \cdots  e_{t-1}$ are paths in $S(H)$, where 
$\lambda(P)=\lambda(P')=t-1$.

As we have noted, our primary focus is on orderly hypertrees. 
Figure~\ref{fig:hypertree+skeleton} illustrates an orderly hypertree and its corresponding skeleton graph.
Orderly hypertrees have the property that their skeleton graph is a tree, something that will be exploited by our learning algorithm for orderly hypertrees in Section~\ref{sec:CC}.

\begin{figure}[t]
  \centering
  \includegraphics[width=0.5\textwidth, scale=0.1]{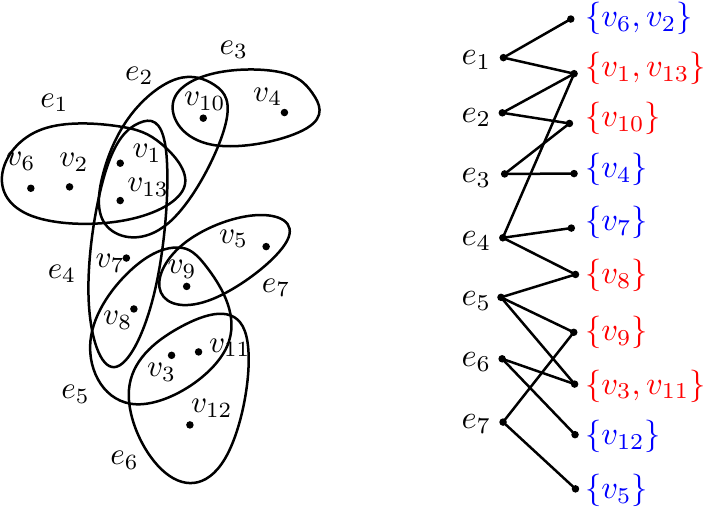}
 \caption{Orderly hypertree (left) and its skeleton graph (right)}
  \label{fig:hypertree+skeleton}
\end{figure}

\begin{remark}\label{rem:uniquepath}
    If $H$ is an orderly hypertree there is a unique path in $S(H)$ joining a specified pair of nodes in $S(H)$. 
    If $P$ joins $[u]$ and $[v]$ in $S(H)$ then $\lambda(P)= d_H(u, v)$.
\end{remark}

In fact, by Proposition~\ref{prop:orderlyPartition} and Remark~\ref{rem:uniquepath}:

\begin{claim} Let $H= (V, E)$ be a hypergraph. Then the following three statements are equivalent. (i)~$H$ is an orderly hypertree. (ii)~$H$ is a hypertree and the colored nodes of $S(H)$ form a partition of $V$. (iii)~$S(H)$ is a tree.
\end{claim}

\begin{remark}
  Let $H$ be an orderly hypertree. Since the distances between all pairs of nodes in $S(H)$ uniquely determine the tree $S(H)$, it follows that---unlike the situation for general hypertrees---the distances between all pairs of vertices in an orderly hypertree $H$ uniquely determine $H$.  
\end{remark}

Thus, given the skeleton graphs of two orderly hypertrees, their isomorphism can be tested in time proportional to the sum of their edge degrees (the size of the skeleton graphs). In particular, orderly hypertree isomorphism  can be tested in time $O(n \Delta \log_{\Delta} m)$, assuming reconstruction can be efficiently implemented.
Similarly, other properties of a given orderly hypertree, such as its diameter (the length of the longest path) or centroid, can be computed from this representation in time proportional to the sum of their edge degrees.

\section{Learning Orderly Hypertrees With $\SP$-Queries}\label{sec:CC}

The literature on learning (conventional) graphs provides at least two approaches demonstrating that trees can be learned with $O(n\log_\Delta n)$ $\SP$-queries. The first approach is due to \citet{Hein}, who proved the weaker claim that phylogenetic trees with bounded degree can be learned with $O(n\log_2 n)$ $\SP$-queries between leaves. \citet{BastideG25} obtained the  $O(n\log_\Delta n)$ bound with a different method, and for general trees with degree bounded by $\Delta$. 
Both used graph separators for the efficient implementation of updates. 

We will establish below that orderly hypertrees of diameter at least three can be learned with $O(n\log_\Delta m)$ $\SP$-queries. To obtain this result 
it would have been natural to try and generalize the apparently simpler approach of
Bastide and Groenland. 
However, Hein's approach has the advantage that it is designed to insert vertices incrementally into an initially empty tree. We will show that this approach can be generalized to an \emph{online}\/ learning algorithm for orderly hypertrees. 
While superficially similar to Hein's approach our algorithm has to contend with complications associated with hyperedges, among which are the facts  that in an orderly hypertree $H$ (i) paths in $H$ that share a hyperedge could be vertex disjoint, and (ii) hyperedges can participate in more than two cliques in the line graph of $H$.

\paragraph{Online Learning of Hypergraphs.}
Let $V^*$ be some universal set of vertices, and $\mathcal H$ a class of connected hypergraphs $H=(V,E)$, where $V \subseteq V^*$. In this setting, for any target hypergraph $H=(V,E) \in \mathcal H$, the set $V$ (or even its size $n$) is \emph{not} known to the learner. Instead, it is presented as a sequence. With each successive vertex $v_{\mathrm next}$, the learner poses a set of distance queries between  $v_{\mathrm next}$ and previously presented vertices, to an oracle. The learner is said to identify $H$ in an online fashion if, for every $i \in [1:n]$,
the distances obtained from the queries associated with the first $i$ vertices completely determine 
the distances between all pairs among these first $i$ vertices. and hence every hypergraph $H’\in \mathcal H$ that is consistent with these distances satisfies  $H’[v_1, \ldots v_i] = H[v_1, \ldots v_i]$.
In particular, after all $n$ vertices have been presented, the distance profile is unique to $H$, among hypergraphs on $V$ within ${\mathcal H}$. Note that the cost of (i.e., the number of queries asked by) a learning algorithm $A$ when identifying a target hypergraph $H$ in a class $\mathcal{H}$ now depends on the sequence in which the vertices in $V$ are presented to $A$. 
In general, the cost of online algorithms can be substantially higher than the cost of offline algorithms. For a fair competitive analysis, one needs to consider the cost with respect to other algorithms that deal with the same presentation.

\cite{FMMZ} showed that  the class of hypergraphs consisting of only two intersecting edges  is hard to learn offline with $\SP$-queries. Given any individual vertex $v\in e_1\cap e_2$  in the target hypergraph $H=(V,\{e_1,e_2\})$, the learner needs to know a pair $(v_1,v_2)\in P_H(e_1)\times P_H(e_2)$ in order to determine that $v$ is not private. As long as no such pair is known, all distances observed will be 1, which leaves open the possibility of all the vertices used in queries so far belonging to a single edge. Thus, \emph{every}\/ online learner for the class of orderly hypertrees (even if constrained to have diameter at least three) must incur a worst-case cost quadratic in the number $h$ of vertices presented before two vertices from two distinct edges have occurred in the presentation.

The algorithm we present below consumes $O(h^2+i\Delta\log_\Delta m)$ queries on any vertex sequence of length $i$, which turns out to be asymptotically optimal.

\subsection{Induced Sub-Skeletons of Orderly Hypertrees}

Our main result is an algorithm that learns orderly hypertrees from $\SP$-queries in an online fashion. For any sequence $(v_1,\ldots,v_i)$ processed by the algorithm, it produces a skeleton graph that is consistent with the \emph{sub-skeleton induced by the target hypertree}\/ on the vertex set $\{v_1,\ldots,v_i\}$---a notion we first need to define formally.
Let $H^*=(V^*, E^*)$ be an orderly hypertree, and let $d^*(u, v)$ denote the length of a shortest path joining $u$ and $v$ in $H^*$.
For any $V \subseteq V^*$, denote by $d^*(V,V)$ the set 
$\{ d^*(u, v) \;|\; u,v \in V \}$. The elements of $d^*(V,V)$ define a unique substructure of $S(H^*)$:

\begin{definition}
    The \emph{sub-skeleton} of $H^*$ induced on $V$ is formed from $S(H^*)$ by (i) replacing all colored nodes by their intersection with $V$, and then (ii) choosing the smallest subtree that contains all of the resulting non-empty colored nodes.
\end{definition}

The length of the path joining two colored nodes in a sub-skeleton, whether they are occupied or not, uniquely determines the hypertree distance between vertices that ultimately occupy those nodes. This allows our learning algorithm to infer full distance information from limited distance queries.
Our online algorithm now proceeds in two phases. In Phase 1, it determines whether all vertices presented so far belong to a single edge. As soon as this is no longer true, Phase 2 starts.

\begin{figure}[t]
  \centering
  \includegraphics[width=0.7\textwidth, scale=0.2]{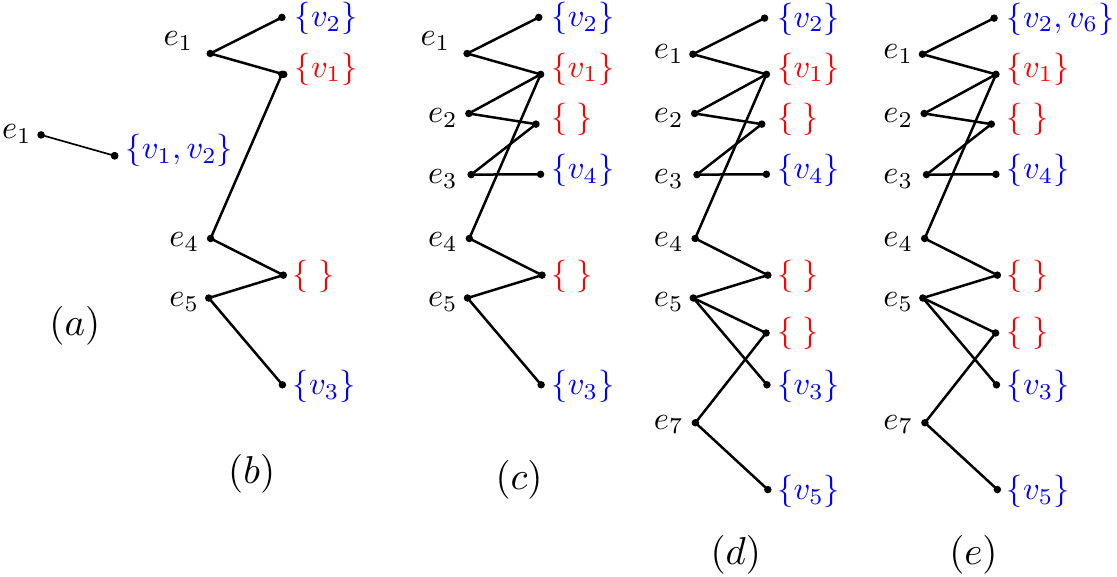}
 \caption{Sub-skeletons of the orderly hypertree from Figure~\ref{fig:hypertree+skeleton} induced on vertex sets $\{v_1, v_{2} \}$ (a), $\{ v_1, v_{2}, v_3 \}$ (b), $\{ v_1, v_{2}, v_3, v_4 \}$ (c), $\{v_1, \ldots, v_5 \}$ (d) and
 $\{v_1, \ldots, v_6 \}$ (e).}
  \label{fig:sub-skeletons}
\end{figure}

\subsection{Phase 1: while there is a one-edge sub-skeleton consistent with all $\SP$-queries so far}

Suppose that vertices are indexed by their position in the insertion sequence. Addition of a new vertex $v_{\mathrm next}$, ${\mathrm next}>1$ involves querying its distance from all vertices $v_j$, $j < {\mathrm next}$. While $d^*(v_j, v_{\mathrm next})= 1$, for all $j < {\mathrm next}$, 
$v_{\mathrm next}$ is added to the single blue node in the skeleton containing all of the vertices $v_j$, $j < {\mathrm next}$ (private vertices of a single edge). The phase ends when it is discovered that $d^*(v_j, v_{\mathrm next})> 1$, for some $v_j$, $j < {\mathrm next}$. The single blue node is identified as the attachment node $N_A$ and the skeleton is updated as described in Phase 2 below.

 For the hypertree illustrated in Figure~\ref{fig:hypertree+skeleton}, Phase 1 ends with the insertion of vertex $v_3$. Figure~\ref{fig:sub-skeletons} (b) illustrates the sub-skeleton at the transition from Phase 1 to Phase 2.

\subsection{Phase 2: the current sub-skeleton has diameter at least two}

The change in the current sub-skeleton resulting from the insertion of a new vertex $v_{\mathrm next}$ is either (i) the expansion of a colored sub-skeleton node, or (ii) the appendage of a chain of new sub-skeleton nodes at some attachment point in the current sub-skeleton.
Figure~\ref{fig:sub-skeletons} (c) (d), and (e) illustrate the sub-skeleton after the insertion of the first four, five and six vertices.
Addition of a vertex $v_{\mathrm next}$ in this phase updates the current sub-skeleton by (i) determining the point of update/attachment (using maximal path separators), and (ii) updating the sub-skeleton at this attachment node, denoted $N_A$. 

\paragraph{Determining the Point of Update/Attachment}

The location of the node $N_A$ is determined by identifying a sequence of paths in the current sub-skeleton. The paths in question all have the skeleton node $[v_1]$ as one endpoint and a blue node (that we denote by 
$[v_{\bot}]$) as the other endpoint. 

For each such path $P$, we need to determine the node on $P$, denoted $N_{\vdash}$, where either (i) vertex $v_{\mathrm next}$ should be inserted, or  
(ii) the path from node $[v_1]$ to node $v_{\mathrm next}$ departs from $P$.
For this we use $\SP$-queries to determine the distance from $[v_{\bot}]$ to both $v_1$ and $v_{\mathrm next}$. These distances (together with the distance from $v_1$ to $v_{\mathrm next}$) determine both the location of node $N_{\vdash}$ (and ultimately node $N_A$).
Note that earlier tree-reconstruction methods use the same observation (cf. \citep{Hein}).

\begin{lemma}\label{lem:sub-skeleton}
The pairwise distances (in $H^*$) between any three vertices, $u, v, {\mathrm and\;} w$ uniquely determine the sub-skeleton induced on these vertices. 
\end{lemma}

\begin{proof}
By Remark~\ref{rem:uniquepath}, there is a unique path from $[u]$ to $[v]$ (resp.,  $[u]$ to $[w]$ and  $[v]$ to $[w]$) in $S(H^*)$. Since the length of this path is $d_{H^*}(u, v)$ (resp., $d_{H^*}(u, w)$ and
$d_{H^*}(v, w))$), the sub-skeleton induced on $\{u, v, w \}$ is uniquely determined.
\end{proof}

Consider the current sub-skeleton, expanded to include $v_{\mathrm next}$ and rooted at $[v_1]$.
Removal of the edges along the path $P$ from
$[v_0]$ to $[v_{\bot}]$ yields a collection of disjoint rooted subtrees.
By Lemma~\ref{lem:sub-skeleton}, we can determine node $N_{\vdash}$, the root of
the subtree in this collection that contains  $[v_{\mathrm next}]$.  

Let $k = d^*(v_1, v_{\mathrm next}) + d^*(v_{\bot}, v_{\mathrm next}) -
d^*(v_1, v_{\bot})$. There are two cases:

($k$ is even)  In this case, $v_{\mathrm next}$ belongs to 
the subtree rooted at the colored node $N_{\vdash}$ at distance $d^*(v_1, v_{\mathrm next})$ from $[v_1]$ on the path between $[v_1]$ to $[v_{\bot}]$. 

($k$ is odd)  In this case, $v_{\mathrm next}$ belongs to 
the subtree rooted at the black node $N_{\vdash}$ at distance $d^*(v_1, v_{\mathrm next})$ from $[v_1]$ on the path between $[v_1]$ to $[v_z]$.

The process continues, by choosing $[v_{\bot}]$ to be a leaf node in 
the subtree rooted at $N_{\vdash}$, until this subtree is reduced to a single node, which is the desired update/attachment node $N_A$.

The choice of node $[v_{\bot}]$ is critical in realizing the desired bound on the number of $\SP$-queries. By  generalizing an observation made by \cite{Hein}, one can show that a greedy choice results in a bound of $\Omega(n \sqrt{m})$ $\SP$-queries; see Appendix~\ref{app:} for details.
The choice (and its associated path) that leads to the efficient (measured in terms of worst-case number of resulting $\SP$-queries) location of the update/attachment point $N_A$ is abstracted as a two player game on $S(H)$, detailed below. 

\paragraph{Skeleton Update}
Lemma~\ref{lem:sub-skeleton} dictates how to update a sub-skeleton at $N_A$.
There are two cases:\\
($k=0$) In this case the new vertex $v_{\mathrm next}$ expands the colored attachment node $N_A$.  \\
($k >0$) In this case a new node $[v_{\mathrm next}]$ is created, linked by a length $k$ alternating chain of black and colored nodes to the attachment node $N_A$. 

If the attachment node $N_A$ is blue (with adjacent black node $N_B$) then $k$ must be even. In this case, $N_A$ becomes red and the vertices currently associated with that node that have distance $k/2 +1$ from $v_{\mathrm next}$ are split off into a new blue node adjacent to $N_B$. This splitting can occur at most $\Delta$ times for any vertex, since each split adds one to the degree of the edge associated with $N_B$. Thus the total number of $\SP$-queries needed to perform the splits in the first $i$ insertions is $O(i \Delta)$. 

In our example, $v_2$ is initially inserted in $[v_1]$, but $[v_1]$ is split on the insertion of $v_3$ (Figure~\ref{fig:sub-skeletons} (a) and (b)). Insertion of $v_3$ and $v_4$  
lead to even length attachments at $[v_1]$ (Figure~\ref{fig:sub-skeletons}(b) and (c)). Following this, inserting $v_5$ leads to an odd length attachment at the black node $e_5$ (Figure~\ref{fig:sub-skeletons}(d)). Next, inserting $v_6$ leads to an expansion of node $[v_2]$ (Figure~\ref{fig:sub-skeletons}(e)).

In our example, the insertions of both $v_2$ and $v_4$ lead to even length attachments at $[v_1]$ (Figure~\ref{fig:sub-skeletons}(center left)). Following this, inserting $v_5$ leads to an odd length attachment at the black node $e_5$ (Figure~\ref{fig:sub-skeletons}(center right)). Next, inserting $v_6$ leads to an expansion of node $[v_2]$ (Figure~\ref{fig:sub-skeletons}(right)).

\subsection{Algorithm Analysis: Correctness and Complexity}

The correctness of the update process follows from Lemma~\ref{lem:sub-skeleton}. 
The process of locating the update/attachment point in the current skeleton terminates since the size of 
the subtree rooted at $N_{\vdash}$
decreases with each successive path $P$.
To obtain an $\SP$-query complexity of $O(\Delta \log_\Delta m)$ for each insertion step, 
the method for selecting node $[v_{\bot}]$ (and the associated path) is crucial. 

\begin{definition}
Let $T=(V,E)$ be a rooted tree.
The \emph{path depth game} on $T$ is played in rounds by two players, $\pi_{\min}$ and $\pi_{\max}$. In each round, (i) $\pi_{min}$ first selects a path $P$ in $T$ and deletes all edges on $P$ from $E$ (leaving $V$ unchanged), and then (ii) $\pi_{max}$ selects one of the thus exposed rooted subtrees $T'$ of $T$ and sets $T:=T'$. The game ends when $T$ consists of a single vertex.

We define the \emph{path depth}\/ of $T $ to be the number of rounds after which this game ends, assuming that $\pi_{\min}$ aims to minimize the number of rounds and $\pi_{\max}$ aims to maximize it.
\end{definition}

Imagine that the path game is played on the current sub-skeleton rooted at node $[v_1]$.
The choice of a path $P$ by $\pi_{\min}$ translates directly into a choice of node $[v_{\bot}]$: simply extend $P$ to any leaf node.
The choice of exposed subtree by $\pi_{\min}$ reflects the choice of an adversary forcing the worst-case behaviour of our algorithm.
It follows that the number of $\SP$-queries required, in the worst case, to locate the update/attachment node $N_A$ in any insertion step is at worst three times the path depth of the current sub-skeleton. 
We show that this depth value is in $\Theta(\Delta\log_\Delta m)$, where $\Delta$ denotes the maximum vertex degree in the target skeleton $S(H^*)$ (and thus the maximum edge degree in $H^*$).

\begin{theorem}\label{thm:pathdepth}
Consider the class of all rooted trees $T=(V,E)$ with $\ell$ leaves and maximum vertex degree at most $\Delta$. The worst-case path depth of such trees is in $\Theta(\Delta\log_\Delta \ell)$.
\end{theorem}

\begin{proof} Let $k=\lceil\log_\Delta \ell \rceil$. 
To prove the lower bound, suppose $\Delta$ is odd. We show that $\Omega(k \Delta)$ rounds are required for the path depth game with a complete rooted tree $T_k$ of uniform out-degree $\Delta$ and depth $k$ (which has $\Delta^i$ nodes on level $i$).  
Let $R_{i,j}$ denote a rooted tree, whose root has $2i +1$ children, each of these the root of a subtree $T_j$. Note that $R_{(\Delta-1)/2, k-1}=T_k$.
Faced with $R_{i,j}$, where $i>0$ and  $j\ge 0$, $P_{\min}$'s choice of a path contains nodes in a most two of the $2i+1$ principal subtrees. This will leave an untouched subtree $R_{i-1, j}$, in the case $i >1$, or $R_{(\Delta-1)/2, j-1}$, if $i=1$ and $j>0$.
It follows by induction on $i$ and $j$ that $\pi_{\max}$ can force $i + j ((\Delta-1)/2)$  rounds. So, starting with $R_{(\Delta-1)/2, k-1} $, the player $\pi_{\max}$ can force at least  $k(\Delta-1)/2$  rounds.

For the upper bound, let $T$ be any rooted tree with $\ell$ leaves and maximum vertex degree  $\le \Delta$. 
We will call any vertex in $T$ $i$-light, if it roots a subtree that has at most $\Delta^{i-1}$ leaves. All other vertices are called $i$-heavy. A $i$-light ($i$-heavy) subtree is a subtree rooted in a $i$-light ($i$-heavy) vertex. The term \emph{$i$-fringe-heavy vertex}\/ refers to any $i$-heavy vertex all of whose children are $i$-light. 

The game is played by $\pi_{\min}$ in $k$ stages.
Stage $i$, $i = k, k-1, \ldots 1$  starts with a $i+1$-light tree and, following $O(\Delta)$ rounds, forces $\pi_{\max}$ to select a $i$-light tree. 
Hence, after $k$ stages, consisting of $O(k\Delta)$ rounds in total, the game ends with a subtree consisting of a single vertex.

Consider the $i$-th stage. For any path $(r,u_1,\ldots,u_s,v)$ from $r$ to a $i$-fringe-heavy vertex $v$, all vertices $u_1,\ldots,u_s$ are $i$-heavy, but not $i$-fringe-heavy. Thus, $T$ has  $\le \Delta$ $i$-fringe-heavy vertices. (Otherwise $T$ would have  $>\Delta\cdot\Delta^{i-1}$ leaves,  contradicting our assumption that the stage begins with an $i+1$-light tree.)
Also, every $i$-heavy vertex lies on a path from $r$ to one of the at most $\Delta$ $i$-fringe-heavy vertices.
$\pi_{\min}$ starts by picking any path $P$ from $r$ to any $i$-fringe-heavy vertex $v$. $\pi_{\max}$ can now choose between any of the subtrees rooted at vertices on $P$. 
Since any such choice results in a subtree with one fewer $i$-fringe-heavy vertex in its interior, it follows that after at most $\Delta$ rounds $\pi_{\max}$ is forced to choose a subtree $T$ that contains no $i$-heavy vertices in its interior.

At this point, either $T$ is $i$-light, which ends the stage, or all of the at most $\Delta$ principle subtrees of $T$ are $i$-light. 
In the latter case, $\pi_{\min}$ can select a path consisting of only a single edge incident to $v$, splitting off a light subtree with each such selection. After no more than $\Delta$ such selections, $\pi_{\max}$ is forced to return a $i$-light tree, ending the stage.
Thus the $i$-th stage ends after  $\le 2\Delta$ rounds. 
\end{proof}

To sum up, we obtain  Theorem~\ref{thm:online} for an algorithm without prior knowledge of $V, n, E, m$, or $\Delta$.

\begin{theorem}\label{thm:online}
    The class of all $n$-vertex orderly hypertrees of diameter at least three can be learned, in an online fashion, using $O( h^2 + i \Delta \log_{\Delta} m)$ $\SP$-queries in total for the first $i$ vertex insertions, where $m$ (resp., $\Delta$) is the (unknown) number of edges
    (resp., maximum edge degree) of the target hypertree, and $h$ is the length of the prefix of the first $i$ insertions that consists of vertices that all belong to a common edge.
\end{theorem}

\subsection{Offline Learning of Orderly Hypertrees}

Given that the target hypertree has diameter at least three, an offline learner with access to $V$ can ask a linear number of queries of the form $(v_1,v_2)$, $(v_1,v_3)$, $(v_1,v_4)$, until a vertex at distance greater than 1 from $v_1$ is found (some such vertex must exist). The learner can then proceed exactly as the online learner does in Phase 2, and identify the target graph without the $h^2$ overhead. This yields the following result, which is realized by an algorithm that knows (only) $V$ in advance.

\begin{theorem}\label{thm:offline}
    The class of all $n$-vertex orderly hypertrees of diameter at least three can be learned, in an offline fashion, using $O( n\Delta \log_{\Delta} m)$ $\SP$-queries, where $m$ (resp., $\Delta$) is the (unknown) number of edges
    (resp., maximum edge degree) of the target hypertree.
\end{theorem}

Recall that dropping the condition on the diameter would give us a lower bound of $\Omega(n^2)$, since orderly hyperstars are hard to learn.
Note, however, that connectivity is not essential for learning to have sub-quadratic cost. 
The additional cost in dealing with a hyperforest with $\le c$ hypertrees is $O(cn)$ queries, since $c$ queries will suffice to determine which of the hypertrees a new vertex belongs to; the remainder of the insertion procedure works as described above.

\subsection{Asymptotic Optimality of Our Algorithm}\label{sec:lowerbound}

Both \citet{10} and
\citet{BastideG24} prove tight lower bounds on tree reconstructions by reduction from certain set partition learning problems. It turns out to be reasonably straightforward to generalize these arguments to hypertree reconstruction if we are only concerned with worst case bounds for deterministic algorithms.

We begin by defining, for any fixed $\Delta \ge 2$ and $k \ge 1$ a \emph{base tree} $T_{\Delta}^k$:  a rooted tree of depth $k$, with vertices
$z_{\langle a\rangle  }$, $a \in \{1, \ldots, \Delta \}^\ell$ on level $\ell$, $0 \le \ell \le k$,
and, for every vertex $z_{\langle i_1, \ldots, i_{\ell - 1}, i_{\ell}\rangle  }$ 
on level $\ell$, $1 \le \ell \le k$, an 
edge $\{ z_{\langle i_1, \ldots, i_{\ell - 1}\rangle  }, z_{\langle i_1, \ldots, i_{\ell - 1}, i_{\ell}\rangle  } \}$ joining $z_{\langle i_1, \ldots, i_{\ell - 1}, i_{\ell}\rangle  }$ to its parent $z_{\langle i_1, \ldots, i_{\ell - 1}\rangle  }$ on level $\ell -1$.

We augment $T_{\Delta}^k$ with vertices from the set $V= \{ v_1, \ldots, v_N \}$ to form a class
${\mathcal H}(\Delta, k, N)$ of  $\Delta^{kN}$ orderly hypertrees with $(\Delta^{k+1} -1)/(\Delta -1) + N$ vertices and $(\Delta^{k+1} -1)/(\Delta -1)$ hyperedges. 
${\mathcal H}(\Delta, k, N)$ contains one hypertree for every partition $\mathcal S$ of $V$ into (possibly empty) sets $S_{\langle a\rangle  }$, where $a \in \{1, \ldots, \Delta \}^k$, where the extreme base tree edge 
$\{ z_{\langle i_1, \ldots, i_{k - 1}\rangle  }, z_{\langle i_1, \ldots, i_{k - 1}, i_{k}\rangle  } \}$ is replaced by the hyperedge
$\{ z_{\langle i_1, \ldots, i_{k - 1}\rangle  }, z_{\langle i_1, \ldots, i_{k - 1}, i_{k}\rangle  } \} \cup S_{\langle i_1, \ldots, i_{k - 1}, i_{k}\rangle  }$,
for every $\langle i_1, \ldots, i_{k - 1}, i_{k}\rangle   \in \{1, \ldots, \Delta \}^k$.

To learn members of ${\mathcal H}(\Delta, k, N)$, it is necessary and sufficient to determine the associated partition $\mathcal S$ of $V$, i.e., to determine, for each $v \in V$, the set $S_{\langle a\rangle  }$ to which $v$ belongs.  

The following adversary strategy shows that to determine, for each $v \in V$, the set $S_{\langle a\rangle  }$ to which $v$ belongs, any $\SP$-query strategy $\mathcal A$ must use at least $(\Delta -1) k$ queries involving each vertex $v \in V$.

We say that vertex $v$ is a \emph{known descendant} of vertex $z_{\langle a\rangle  }$ on level $\ell < k$ if $d_H(z_{\langle a\rangle  }, v) = k - \ell$. 
Once it has been established that vertex $v$ is a known descendant of vertex $z_{\langle a\rangle  }$, the adversary can restrict attention to forcing algorithm $\mathcal A$ to distinguish hypergraphs among those that satisfy this constraint. 
Note that to start each vertex $v$ is a known descendant of the root vertex $z_{\langle\rangle  }$ only.

In general, the adversary strategy maintains an invariant property that, for every vertex $v \in V$,\\ 
(i) $v$ is a known descendant of some vertex $z{\langle a\rangle  }$ on some level $\ell < k$ and has participated in $q < \Delta -1$ $\SP$-queries with other known descendants of vertex $z_{\langle a\rangle  }$, and\\
(ii) $v$ has participated in at least $(\Delta -1) \ell $ $\SP$-queries with other vertices.

There are several cases to consider for the adversary in responding to the next $\SP$-query of algorithm $\mathcal A$:
\begin{itemize}
    \item {$\mathbf{d_H(z_{\langle b\rangle  }, z_{\langle c\rangle  })}$?} In this case, the adversary has no choice in its response;
    \item $\mathbf{d_H(v, z_{\langle c\rangle  })}$? In this case, the adversary responds $d_H(v, z_{\langle c\rangle  }) = k- \ell + d_H(z_{\langle a\rangle  }, z_{\langle c\rangle  })$.This says, in effect, that $v$ is \emph{not} a descendent of $z_{\langle c\rangle  }$;
    \item $\mathbf{d_H(v, v')}$? In this case, the adversary responds 
    $d_H(v, v') = 2(k- \ell')$, if the lowest common known ancestor of $v$ and $v'$ is on level $\ell'$.
    In particular, it responds $d_H(v, v') = 2(k- \ell)$, if $v'$ is a known descendant of
    vertex $z_{\langle a\rangle  }$. This says, in effect, that $v$ and $v'$ are \emph{not} descendents of the same child of $z_{\langle a\rangle  }$.

\end{itemize}

In the event that the next $\SP$-query is the $(\Delta -1)$-th
$\SP$-query involving $v$ with other known descendants of vertex $z{\langle a\rangle  }$, at least one of the children of $z{\langle a\rangle  }$ has none of these as a known descendant. The adversary continues, having declared $v$ to be a known descendant of one such child of $z_{\langle a\rangle  }$.

\begin{claim}\label{clm:lowerbound}
    The adversary strategy forces any given algorithm to make at least $N (\Delta -1) k/2 $ $\SP$-queries in order to learn members of ${\mathcal H}(\Delta, k, N)$.
\end{claim}

\begin{proof}
    It is straightforward to confirm that the adversary strategy maintains the desired invariant. It follows immediately from the invariant that until every vertex $v \in V$ has been involved with $\Delta -1$ $\SP$-queries after having become a known descendant of some vertex $z_{\langle a\rangle  }$ on level $k-1$, there are at least two hypertrees in ${\mathcal H}(\Delta, k, N)$ that are consistent with the results of all queries to date. Since every $\SP$-query invoves at most two vertices in $V$, the claim follows.
\end{proof}

Provided $N$ is at least $(\Delta^{k+1} -1)/(\Delta -1)$, hypertrees in the class ${\mathcal H}(\Delta, k, N)$ have 
$n = (\Delta^{k+1} -1)/(\Delta -1) + N$ vertices, $m = (\Delta^{k+1} -1)/(\Delta -1)$ edges and $k = \Theta(\log_{\Delta} m)$ levels. Consequently, we have demonstrated the following theorem, showing that the upper bound in Theorem~\ref{thm:offline} is asymptotically tight.

\begin{theorem}
For any fixed $\Delta \ge 3$, $m \ge 1$, and $n \ge 2m$
any deterministic offline algorithm for learning orderly hypertrees, with diameter at least three, degree at most $\Delta$, $n$ vertices and $m$ hyperedges, requires $\Omega(n \Delta \log_{\Delta} m)$ $\SP$-queries in the worst case.
\end{theorem}

\begin{remark}
    As we have already observed, the $\Theta(h^2)$ overhead in our online algorithm is unavoidable. Thus our online algorithm for learning orderly hypertrees is also optimal, up to a constant factor.
\end{remark}

\section{Learning Orderly Hypertrees with Bounded Distance Queries}

Algorithms for learning hypergraph families using $\SP$-queries are based on the assumption that exact distance information, between all vertex pairs, is available from an oracle. 
In some contexts, such as evolutionary tree reconstruction, it may be more reasonable to assume that such information is only available for ``nearby'' vertex pairs. 
To model such situations, we consider \emph{bounded distance queries} that provide accurate distance information only when the distance between specified species is less than some fixed threshold.

\begin{definition}
    A ($\mathrm{dist}_{\le d}$-query)
for a given pair of vertices, returns their distance, if the distance is at most $d$, and ``$>\!d$'' otherwise.
\end{definition}

It is interesting to note that the orderliness property makes it possible to construct efficient algorithms for learning general connected hypergraphs, not just hypertrees, using 
$\mathrm{dist}_{\le 1}$-queries:

\begin{theorem} \label{thm:dist1alg}
    Let $\mathcal{H}$ be the class of all 
connected orderly hypergraphs of order $n$ and diameter at least 3, each hyperedge of which contains at least one private vertex. 
    $O(mn)$ $\mathrm{dist}_{\le 1}$-queries suffice to
    learn any target hypergraph $H\in \mathcal{H}$, 
    where $m$ is the number of hyperedges in $H$.
\end{theorem}

\begin{proof} The claim is witnessed by the following algorithm for learning any $H\in \mathcal{H}$: 

(1) Initialize $i=0$, $V^*=V$ and repeat the following until $V^*=\emptyset$:
\begin{itemize}
    \item[(a)] Select a vertex $v_i\in V^*$ and ask a total of $n-1$ $\mathrm{dist}_{\le 1}$-queries of the form $\{v_i,v'\}$ -- one for each vertex $v'\in V$ with $v'\ne v_i$. 
    \item[(b)] Let $V_i\subseteq V$  be the set of vertices $v'$ for which $d_H(v_i, v')=1$.
    (Note that $V_i$ is the union of all edges in $H$ containing $v_i$.) 
    Set $V^*:= V^*\setminus V_i$ and $i:=i+1$.
\end{itemize}
Since $H$ is orderly, the sub-hypergraph induced by $V_i$ is an orderly hyperstar $S_i$ (possibly just a single edge), and hyperedges of $H$ belong to at most one such hyperstar . For each edge $e$ in $H$, since $e$ has at least one private vertex, there must exist an $i$ such that $e\subseteq V_i$. Thus every hyperedge of $H$ belongs to exactly one hyperstar $S_i$.. 

(2) Since $H$ is connected with diameter at least 3, each set $V_i$ intersects with at least one set $V_j$, $j\ne i$. Thus, for each set $V_i$:
\begin{itemize}
    \item[(a)] Select a vertex $w_{i}\in V_i$ that belongs to some set $V_j$, $j\ne i$.  For each vertex  $w\in V_i$, ask one $\mathrm{dist}_{\le 1}$-queries of the form $\{w_{i},w\}$.

  \item[(b)] If all queries are answered $1$, then $V_i$ is a single edge. Otherwise, the vertices in $V_i$ form a hyperstar $H_i$ within $H$; moreover, the vertex $w_i$ has maximum eccentricity in $H_i$. (Note that the intersections of $V_i$ with any $V_j$ are known at this point.) 

\end{itemize}

(3) The hyperedges in each $H_i$ can be learned with at most $O(n_i m_i)$ $\mathrm{dist}_{\le 1}$ queries, where $n_i=|V_i|$ is the number of vertices in $H_i$, and $m_i$ is the number of edges in $H_i$. To do so, note that the edge $e$ in $H_i$ that contains $w_i$ is already determined by the queries in (2a). 
   The other edges of $H_i$ are learned by picking any vertex $u\in V_i\backslash e$ and asking $\mathrm{dist}_{\le 1}$-queries of the form $\{u,w\}$ for each vertex $w\in V_i$. Continuing this process, each time picking a vertex $u\in V_i$ that does not belong to the previously learned edges of $H_i$, all edges of $H_i$ are learned.  
   The total number of $\mathrm{dist}_{\le 1}$-queries in this step is in $O(\sum_{i} m_i n_i)$ which can be bounded by $O(n\sum_i m_i)=O(m n)$.\\
Hence the total number of $\mathrm{dist}_{\le 1}$-queries consumed in order to learn $H$ is in  $O(mn)$.
\end{proof} 

\begin{remark} 
    As  noted in Remark~\ref{rem:SPhard}, dropping either (i) the condition on diameter at least 3, or (ii) the orderliness condition, would admit hypertrees that require $\Omega(n^2)$ $\SP$-queries. 
    If the private vertex condition alone is dropped it is straightforward to see that 
    $\Omega(n^2)$ $\mathrm{dist}_{\le 1}$-queries could be required, as witnessed by the class of hyperpaths of diameter  $\ge 3$: 
\end{remark}

\begin{lemma}\label{lem:paths}
    Let $\mathcal{H}$ be the class of all hyperpaths (which are, in particular, orderly hypertrees) of diameter at least three. Learning a target hypergraph in $\mathcal{H}$ requires $\Omega(n^2)$ $\mathrm{dist}_{\le 1}$-queries
     in the worst case. By contrast, the worst-case number of $\SP$-queries  required to 
learn any target hypergraph in $\mathcal{H}$ is in $\Theta(n)$. 
\end{lemma}
\begin{proof}
     The $\Theta(n)$ upper bound was proven by \citet{FMMZ}). 
The quadratic lower bound on the $\mathrm{dist}_{\le 1}$-query complexity can be derived as follows. For any $n\ge 4$, consider a hyperpath $P_n$ (Figure \ref{fig:5path}) of length three with two edges $e_1$ and $e_3$ of cardinality $n/2$ and one edge $e_2$ of cardinality $2$ with no private vertices. 

Even if the vertices in $e_1$ and $e_3$ are identified as such, identifying the vertices in $e_2$ forces any algorithm to determine the unique pair from $e_1 \times e_3$ that has distance 1. Clearly, an adversary can force a learner to query all pairs in $e_1 \times e_3$ before the correct one is discovered.
\end{proof}

\begin{figure}[t]
  \centering
  \includegraphics[width=0.5\textwidth, scale=0.1]{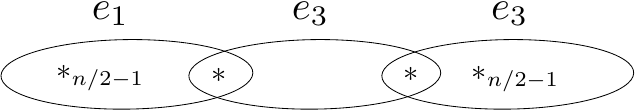}
 \caption{Hyperpath $P_n$ that requires $\Theta(n^2)$ $\mathrm{dist}_{\le 1}$-queries.}
  \label{fig:5path}
\end{figure}

As we have seen, the assumption that hypertrees in a given class have hyperedges all of which have private vertices is essential for the existence of a learning algorithm of sub-quadratic complexity using $\mathrm{dist}_{\le 1}$-queries.
However, if the bound on distance queries is increased, even slightly, it is possible to get the same $O(n m)$ upper bound on query complexity even if the private vertex assumption is dropped. 

\begin{theorem} \label{thm:dist2alg}
    Let $\mathcal{H}$ be the class of all orderly hypertrees 
    of order $n$ and diameter at least 3.
    $O(mn)$ $\mathrm{dist}_{\le 2}$-queries suffice to 
    learn any target hypergraph $H\in \mathcal{H}$,
    where $m$ is the number of hyperedges in $H$.
    
\end{theorem}


\begin{proof}
    The basic idea is very similar to the $\mathrm{dist}_{\le 1}$-query strategy described in the proof of Theorem~\ref{thm:dist1alg}:\\
(1) Initialize $i=0$, $V^* = V$, and  $V^+=\{ v \}$, for some arbitrary vertex $v \in V$, and repeat the following until $V^*=\emptyset$:
\begin{itemize}
    \item[(a)] Select a vertex $v_i\in V^+$ and ask a total of $n-1$ $\mathrm{dist}_{\le 2}$-queries of the form $\{v_i,v'\}$ -- one for each vertex $v'\in V$ with $v'\ne v_i$. 
    \item[(b)] Let $V^1_i$  be the set of vertices $v' \in V$ satisfying $d_H(v_i, v')= 1$.
    (Note that $V_i$ is the union of all edges in $H$ containing $v_i$.) 
    Set $V^*:= V^*\setminus V_i$ 
    \item[(c)]
    Let $V^2_i$  be the set of vertices $v' \in V^*$ satisfying $d_H(v_i, v')= 2$.
    Set $V^+ = (V^+ \setminus \{ v_i \}) \cup V^2_i$
    and $i:=i+1$.
\end{itemize}
Since $H$ is orderly, the sub-hypergraph induced by $V_i$ is an orderly hyperstar $S_i$ (possibly just a single edge), so hyperedges of $H$ belong to at most one such hyperstar. 
Furthermore, since $H$ is connected, every vertex $v$ belongs to at least one such hyperstar $S_i$, and, by construction, each hyperstar $S_i$, $i > 1$, contains at least one vertex $v^j_i$ belonging to some hyperstar $S_j$, $j < i$. So, 
there is a path from $v$ to $v_1$ using edges in stars in $\{ S_j \;|\; j\le i \}$.
Since $H$ is acyclic, it follows that every edge of $H$ belongs to exactly one hyperstar $S_i$.

(2) We can now  learn $S_i$ with at most $O(n_i m_i)$ $\mathrm{dist}_{\le 1}$ queries, where $n_i=|V_i|$ is the number of vertices in $S_i$, and $m_i$ is the number of edges in $S_i$. To do so, note that the edge $e$ in $S_i$ that contains $v^j_i$ is
just $\{ v \in V_i \;|\; d_H(v^j_i, v) =1 \}$.  
So $e$ can be determined using $|V_i|$ $\mathrm{dist}_{\le 2}$-queries.

   The other edges of $S_i$ are learned by picking any vertex $u\in V_i\backslash e$ and asking $\mathrm{dist}_{\le 2}$-queries of the form $\{u,w\}$ for each vertex $w\in V_i$. Continuing this process, each time picking a vertex $u\in V_i$ that does not belong to the previously learned edges of $S_i$, all edges of $S_i$ are learned.  
   The total number of $\mathrm{dist}_{\le 2}$-queries in this step is in $O(\sum_{i} m_i n_i)$ which can be bounded by $O(n\sum_i m_i)=O(m n)$.

Hence the total number of $\mathrm{dist}_{\le 2}$-queries consumed in order to learn $H$ is in  $O(mn)$.    
\end{proof}


As it happens the upper bound established in Theorem~\ref{thm:dist1alg} is tight, even for algorithms that use
$\mathrm{dist}_{\le d}$-queries, for any fixed $d \ge 1$ (and even when restricting the hypothesis class to hypertrees).
To establish this lower bound consider the following modification of the adversary argument from 
Section~\ref{sec:lowerbound}.

Begin with the binary base tree $T^k_2$. 
First, add to each of its edges 
$\{ z_{\langle i_1, \ldots i_{\ell-1}\rangle  }, 
z_{\langle i_1, \ldots i_{\ell-1}, i_{\ell}\rangle  } \}$,
$1 \le \ell \le k$,
a new private vertex $z^p_{\langle i_1, \ldots i_{\ell-1}, i_{\ell}\rangle  }$.
Then attach to each of its $2^k$ leaves $z_{\langle a\rangle  }$ a chain of $d$ hyperedges (each with a private vertex):
$\{ z_{\langle a\rangle  }, z^p_{\langle a\rangle  , 1}, z_{\langle a\rangle  , 1} \}, 
\{ z_{\langle a\rangle  , 1}, z^p_{\langle a\rangle  , 2}, z_{\langle a\rangle  , 2} \}, \ldots\\
\{ z_{\langle a\rangle  , d-1}, z^p_{\langle a\rangle  , d}, z_{\langle a\rangle  , d} \}$. 
As in the earlier proof we consider all $2^{kN}$ possible partitions $\mathcal S$ of the vertex set $V= \{ v_1, \ldots, v_N \}$ into $2^k$ (possibly empty) sets; $S_{\langle a\rangle  }$, $a \in \{ 1, 2 \}^N$. 
The class of orderly hypertrees ${\mathcal H}(2, k, N, d)$
 contains one hypertree for every such partition $\mathcal S$, where the extreme hyperedge 
$\{ z_{\langle a\rangle  , d-1}, z^p_{\langle a\rangle  , d}, z_{\langle a\rangle  , d} \}$
is replaced by the hyperedge
$\{ z_{\langle a\rangle  , d-1}, z^p_{\langle a\rangle  , d}, z_{\langle a\rangle  , d} \}
\cup S_{\langle a\rangle  }$,
for every
$a \in \{1, 2 \}^k$.

It is easy to see that vertex $v \in V$ is confirmed to belong to $S_{\langle a \rangle}$ if and only if some $\mathrm{dist}_{\le d}$-query responds $d_H(v, z_{\langle a \rangle, j})= d-j+1 $, for some $1 \le j \le d$. An adversary can clearly delay such a confirmation until the last possible choice for $\langle a \rangle$, independently of the queries involving all other elements of $V$. It follows that in the worst case 
at least $2^k$ $\mathrm{dist}_{\le d}$-queries, involving $v$ alone among vertices in $V$, must be made in order to fix the partition $\mathcal S$, and uniquely determine a hypertree in
${\mathcal H}(2, k, N, d)$.
To summarize:

\begin{theorem}
For any fixed $d \ge 1$, $m \ge 1$, and $n \ge 2m$,
any deterministic offline algorithm for learning orderly hypertrees, with diameter at least three, $n$ vertices, and $m$ hyperedges, each of which contains at least one private vertex, requires $\Omega(n m)$ $\mathrm{dist}_{\le d}$-queries in the worst case.
\end{theorem}

\section{Conclusions}

This paper has focused on query learning of a broad class of hypertrees that we call orderly hypertrees. 
The position of this class in the Fagin hierarchy of acyclic hypergraphs makes it potentially relevant to database theory; in particular, our results show that, compared to other members of the Fagin hierarchy, orderly hypertrees form the broadest class of acyclic hypergraphs for which non-exhausitive  $\SP$-query learning is possible.
We have provided a provably optimal \emph{online}\/ $\SP$-query algorithm for learning a orderly hypertrees of diameter at least three, which can be transformed into an optimal offline algorithm for the same class. 
We have also studied a natural model of restricted distance queries that provide accurate distance information for only those vertex pairs whose separation is less than some specified threshold. 
In this model, we again exploit the orderliness property to demonstrate asymptotically tight bounds on query complexity.

Our $\SP$-query algorithm for online learning of orderly hypertrees is reminiscent of the incremental approach used by \citet{Hein} for conventional trees. 
An alternative approach, generalizing the method of \citet{BastideG25} for reconstructing conventional trees, does not lend itself to online learning. 
However, in addition to its simplicity, it may have other potential advantages, including application to a broader class of hypertrees, allowing general edge weights, or yielding improved bounds for specific subclasses of orderly hypertrees. 
It would also open up the possibility of adopting other restricted query modes as described by \citet{10}, such as 
$\varepsilon$-approximate queries. 
We leave the further exploration of this alternative approach for future consideration.

\acks{
S.\ Fallat was supported in part by an NSERC Discovery Grant, application no.\ RGPIN-2019-03934.\\
D.\ Kirkpatrick was supported through the NSERC Discovery Grants program, under grant no. 22R83583.\\
K. Khodamoradi was supported through the NSERC Discovery Grants program, application no.\ RGPIN-2024-06360.\\
S.\ Zilles was supported through a Canada CIFAR AI Chair at the Alberta Machine Intelligence Institute (Amii), through an NSERC Canada Research Chair, through the New Frontiers in Research Fund  under grant  no.\ NFRFE-2023-00109 and through the NSERC Discovery Grants program under application no.\ RGPIN-2017-05336. 
}

\bibliography{references}

\begin{thebibliography}{14}
\providecommand{\natexlab}[1]{#1}
\providecommand{\url}[1]{\texttt{#1}}
\expandafter\ifx\csname urlstyle\endcsname\relax
  \providecommand{\doi}[1]{doi: #1}\else
  \providecommand{\doi}{doi: \begingroup \urlstyle{rm}\Url}\fi

\bibitem[Abrahamsen et~al.(2016)Abrahamsen, Bodwin, Rotenberg, and St{\"{o}}ckel]{AbrahamsenBRS16}
Mikkel Abrahamsen, Greg Bodwin, Eva Rotenberg, and Morten St{\"{o}}ckel.
\newblock Graph reconstruction with a betweenness oracle.
\newblock In \emph{Proceedings of the 33rd Symposium on Theoretical Aspects of Computer Science {(STACS)}}, pages 5:1--5:14, 2016.

\bibitem[Bastide and Groenland(2024)]{BastideG24}
Paul Bastide and Carla Groenland.
\newblock Optimal distance query reconstruction for graphs without long induced cycles.
\newblock arXiv cs.DS 2306.05979, 2024.

\bibitem[Bastide and Groenland(2025)]{BastideG25}
Paul Bastide and Carla Groenland.
\newblock Tight distance query reconstruction for trees and graphs without long induced cycles.
\newblock \emph{Random Structures and Algorithms}, 66\penalty0 (4), 2025.

\bibitem[Beerliova et~al.(2006)Beerliova, Eberhard, Erlebach, Hall, Hoffmann, Mihal{\'{a}}k, and Ram]{Beerliova}
Zuzana Beerliova, Felix Eberhard, Thomas Erlebach, Alexander Hall, Michael Hoffmann, Mat{\'{u}}s Mihal{\'{a}}k, and L.~Shankar Ram.
\newblock {Network discovery and verification}.
\newblock \emph{IEEE Journal on Selected Areas in Communications}, 24\penalty0 (12):\penalty0 2168--2181, 2006.

\bibitem[Brandstädt et~al.(1999)Brandstädt, Le, and Spinrad]{Brandstadt99}
Andreas Brandstädt, Van~Bang Le, and Jeremy~P. Spinrad.
\newblock \emph{Graph Classes: A Survey}.
\newblock Society for Industrial and Applied Mathematics, 1999.
\newblock \doi{10.1137/1.9780898719796}.
\newblock URL \url{https://epubs.siam.org/doi/abs/10.1137/1.9780898719796}.

\bibitem[Brodal et~al.(2001)Brodal, Fagerberg, Pedersen, and {\"O}stlin]{Brodal2001}
Gerth~St{\o}lting Brodal, Rolf Fagerberg, Christian N.~S. Pedersen, and Anna {\"O}stlin.
\newblock The complexity of constructing evolutionary trees using experiments.
\newblock In \emph{ICALP}, volume 2076 of \emph{Lecture Notes in Computer Science}, pages 140--151. Springer, 2001.
\newblock \doi{10.1007/3-540-48224-5_12}.

\bibitem[Fagin(1983)]{Fagin1983Degrees}
Ronald Fagin.
\newblock Degrees of acyclicity for hypergraphs and relational database schemes.
\newblock \emph{Journal of the ACM (JACM)}, 30\penalty0 (3):\penalty0 514--550, 1983.
\newblock \doi{10.1145/2402.322390}.

\bibitem[Fallat et~al.(2024)Fallat, Maliuk, Mojallal, and Zilles]{FMMZ}
Shaun Fallat, Valerii Maliuk, Seyed~Ahmad Mojallal, and Sandra Zilles.
\newblock Learning hypertrees from shortest path queries.
\newblock In \emph{Proceedings of the 35th International Conference on Algorithmic Learning Theory {(ALT)}}, pages 1--16, 2024.

\bibitem[Hein(1989)]{Hein}
Jotun~J. Hein.
\newblock {An optimal algorithm to reconstruct trees from additive distance data}.
\newblock \emph{Bulletin of Mathematical Biology}, 51\penalty0 (5):\penalty0 597--603, 1989.

\bibitem[Janardhanan(2017)]{Janardhanan17}
Mano~Vikash Janardhanan.
\newblock Graph verification with a betweenness oracle.
\newblock In \emph{Proceedings of the International Conference on Algorithmic Learning Theory {(ALT)}}, pages 238--249, 2017.

\bibitem[Kannan et~al.(2015)Kannan, Mathieu, and Zhou]{9}
Sampath Kannan, Claire Mathieu, and Hang Zhou.
\newblock Near-linear query complexity for graph inference.
\newblock In \emph{Proceedings of the 42nd International Colloquium on Automata, Languages, and Programming {(ICALP)}}, pages 773--784, 2015.

\bibitem[Kannan et~al.(1996)Kannan, Lawler, and Warnow]{Kannan1996}
Sampath~K. Kannan, Eugene~L. Lawler, and Tandy~J. Warnow.
\newblock Determining the evolutionary tree using experiments.
\newblock \emph{Journal of Algorithms}, 21\penalty0 (1):\penalty0 26--50, 1996.

\bibitem[King et~al.(2003)King, Zhang, and Zhou]{10}
Valerie King, Li~Zhang, and Yunhong Zhou.
\newblock On the complexity of distance-based evolutionary tree reconstruction.
\newblock In \emph{Proceedings of the 14th Annual {ACM-SIAM} Symposium on Discrete Algorithms {(SODA)}}, pages 444--453, 2003.

\bibitem[Reyzin and Srivastava(2007)]{Reyzin}
Lev Reyzin and Nikhil Srivastava.
\newblock Learning and verifying graphs using queries with a focus on edge counting.
\newblock In \emph{Proceedings of the 18th International Conference on Algorithmic Learning Theory {(ALT)}}, pages 285--297. Springer, 2007.

\end{thebibliography}

\appendix


\section{On Choosing Path Endpoints}\label{app:}

We remarked in the main body that the choice of node $[v_{\bot}]$ is critical in realizing the desired bound on the number of $\SP$-queries. A greedy choice can be inefficient:

\setcounter{theorem}{23}
\begin{proposition}
    The online learning algorithm desribed in Section~\ref{sec:CC}  chooses a sequence of paths all of which have the skeleton node $[v_1]$ as one endpoint and a blue node 
    $[v_{\bot}]$ as the other endpoint. 
    Suppose $[v_{\bot}]$ is always chosen to be the most distant vertex from $[v_1]$ in the subtree containing the vertex $v_{\mathrm next}$. Then, to locate an attachment node requires $\Omega(\sqrt{\ell})$  $\SP$-queries in the worst case, where $\ell$ is the number of leaves in the current sub-skeleton.
\end{proposition}

\begin{proof} Our argument is a generalization of a similar claim made by \cite{Hein} about their algorithm for conventional trees.

       \begin{wrapfigure}{r}{0.25\textwidth}
        \centering
        \begin{tikzpicture}[every node/.style={circle, inner sep=0pt, minimum size=1.5mm, fill}, 
        rednode/.style ={circle, inner sep=0pt, minimum size=1.5mm, fill, fill=green, draw=green}, scale = 0.7]

\node(r11)[label = {}] at (-1.5,-1.5){}; 
\node(r12)[label = {}] at (-1.5,-1){};   
\node(r13)[label = {}] at (-1.5,-.5){};  
\node(r14)[label = {}] at (-1.5,-0){};   
\node(r15)[label = {}] at (-1.5,0.5){};  
\node(r16)[label = {}] at (-1.5,1){};    
\node(r17)[label = {2}] at (-1.5,1.5){};  

\node(r21)[label = {}] at (-1,-1.5){};   
\node(r22)[label = {}] at (-1,-1){};     
\node(r23)[label = {}] at (-1,-0.5){};   
\node(r24)[label = {}] at (-1,-0){};     
\node(r25)[label = {}] at (-1,0.5){};    
\node(r26)[label = {4}] at (-1,1){};      

\node(r31)[label = {}] at (-0.5,-1.5){}; 
\node(r32)[label = {}] at (-0.5,-1){};   
\node(r33)[label = {}] at (-0.5,-0.5){}; 
\node(r34)[label = {}] at (-0.5,-0){};   
\node(r35)[label = {6}] at (-0.5,0.5){};  

\node(r41)[label = {}] at (0,-1.5){};    
\node(r42)[label = {}] at (0,-1){};      
\node(r43)[label = {}] at (0,-0.5){};    
\node(r44)[rednode, label = {}] at (0,-0){};      

\node(r51)[label = {}] at (0.5,-1.5){};  
\node(r52)[label = {}] at (0.5,-1){};    
\node(r53)[label = {5}] at (0.5,-0.5){};  

\node(r61)[label = {}] at (1,-1.5){};    
\node(r62)[label = {3}] at (1,-1){};      

\node(r71)[label = {1}] at (1.5,-1.5){};  

\draw[] (r11) -- (r12) -- (r13) -- (r14) -- (r15) -- (r16) -- (r17);
\draw[] (r22) -- (r23) -- (r24) -- (r25) -- (r26);
\draw[] (r33) -- (r34) -- (r35);

\draw[] (r11) -- (r21) -- (r31) -- (r41) -- (r51) -- (r61) -- (r71);
\draw[] (r12) -- (r22) -- (r32) -- (r42) -- (r52) -- (r62);
\draw[] (r23) -- (r33) -- (r43) -- (r53);
\draw[] (r34) -- (r44);

    \end{tikzpicture}
    \caption{
    }\label{fig:inefficient}
    \end{wrapfigure}
    Suppose the current sub-skeleton is a binary tree of the 
    form shown in Figure~\ref{fig:inefficient}, where $[v_1]$ is the node in the lower left corner.
    Furthermore, suppose that $v_{\mathrm next}$  has to be inserted at the green leaf.
    If $[v_{\bot}]$ is always chosen to be the most distant vertex from $[v_1]$ in the subtree containing the vertex $v_{\mathrm next}$, 
    then the leaf labeled $i$ in the figure is a candidate to be chosen as $[v_{\bot}]$ in the $i$-th step. Scaling up this example, it follows that $\Theta(\sqrt{\ell})$ $\SP$-queries could be used in the worst case. 
\end{proof}

\end{document}